\documentclass[conference]{IEEEtran}
\IEEEoverridecommandlockouts
\usepackage{cite}
\usepackage{amsmath,amssymb,amsfonts}
\usepackage{algorithmic}
\usepackage{graphicx}
\usepackage{textcomp}
\usepackage{xcolor}
\def\BibTeX{{\rm B\kern-.05em{\sc i\kern-.025em b}\kern-.08em
    T\kern-.1667em\lower.7ex\hbox{E}\kern-.125emX}}

\usepackage{graphicx}
\usepackage{comment}
\usepackage{amsmath,amssymb,amsthm}
\usepackage{mathtools}
\usepackage{color}
\usepackage{url}
\usepackage{hyperref}
\usepackage{thm-restate}
\usepackage{xpunctuate}
\usepackage[ruled,vlined]{algorithm2e}

\newcommand{\VS}{\vspace{1ex}}
\providecommand{\cond}{\,\vert\,}

\providecommand{\cond}{\,\vert\,}                               % for conditional probabilities
\DeclarePairedDelimiterX{\scalp}[2]{\langle}{\rangle}{#1,#2}    % scalar product
\DeclarePairedDelimiterX{\infdivx}[2]{(}{)}{%
  #1\;\delimsize\|\;#2}
\newcommand{\eg}{e.g\xperiod}

\newcommand{\ie}{i.e\xperiod}

\newcommand\etc{etc\xperiod}

\newtheorem{theorem}{Theorem}

\begin{document}

\title{Game-theoretic distributed learning of generative models for heterogeneous data collections \\
\thanks{
D.S. was supported by the German Federal Ministry for Economic Affairs and Climate Action (BMWK) project 01MN23021A. B.F. is grateful for support by the CTU institutional support (future fund). The authors would like to thank the Center for Information Services and HPC (ZIH) at TU Dresden for providing computing resources.

The manuscript is accepted for publishing at the 2025 Symposium on Federated Learning and Intelligent Computing Systems (FLICS 2025).

\copyright~2025 IEEE. Personal use of this material is permitted. Permission from IEEE must be obtained for all other uses, in any current or future media, including reprinting/republishing this material for advertising or promotional purposes, creating new collective works, for resale or redistribution to servers or lists, or reuse of any copyrighted component of this work in other works.
	}
}

\author{
	\IEEEauthorblockN{Dmitrij Schlesinger\IEEEauthorrefmark{1}, Boris Flach\IEEEauthorrefmark{2}}
\IEEEauthorblockA{\IEEEauthorrefmark{1}\textit{Faculty of Computer Science},
\textit{Dresden University of Technology}, 
Dresden, Germany \\
Email: dmytro.shlezinger@tu-dresden.de}
\IEEEauthorblockA{\IEEEauthorrefmark{2}\textit{Department of Cybernetics}, 
	\textit{Czech Technical University in Prague}, 
	Prague, Czech Republic \\
	Email: flachbor@fel.cvut.cz}
}

\maketitle

\begin{abstract} 
One of the main challenges in distributed learning arises from the difficulty of handling heterogeneous local models and data. In light of the recent success of generative models, we propose to meet this challenge by building on the idea of exchanging synthetic data instead of sharing model parameters. Local models can then be treated as ``black boxes'' with the ability to learn their parameters from data and to generate data according to these parameters. Moreover, if the local models admit semi-supervised learning, we can extend the approach by enabling local models on different probability spaces. This allows to handle heterogeneous data with different modalities. We formulate the learning of the local models as a cooperative game starting from the principles of game theory. We prove the existence of a unique Nash equilibrium for exponential family local models and show that the proposed learning approach converges to this equilibrium. We demonstrate the advantages of our approach on standard benchmark vision datasets for image classification and conditional generation.
\end{abstract}


\section{Introduction}

The perhaps most popular branch of distributed learning is Federated Learning (FL) \cite{pmlr-v54-mcmahan17a}. It assumes several {\em local models}, each one having access to its own private data. The goal is to learn them jointly without sharing the local data. Originally, \cite{pmlr-v54-mcmahan17a} proposed to learn a global model, assuming that the local models are just replicas of the global one. The global model is then learned by averaging the parameters of the local models after each local update step.

The global nature of the learned model is, however, a severe restriction, which usually contradicts practical requirements, \eg if different local models are deployed on different hardware. Besides, it is often desired that the local models specialize on their own data but to some extent keep being generalizable to all data. That is why later works  relaxed this restriction by allowing different models, which was referred to as heterogeneous FL or, in other works, as personalized FL \cite{DBLP:conf/nips/0001MO20}.

The majority of works dealing with model heterogeneity employ parameter sharing. The general idea is to use models which can differ to some extent but still share some common elements or properties. For example, \cite{DBLP:conf/iclr/Diao0T21} proposes a system of local models with overlapping parameter sets, so that the parameters shared by different local models can be averaged. The authors of \cite{10.1145/3534678.3539384} propose a compositional architecture, where the local models are composed of shared ``building blocks''. In \cite{DBLP:conf/icml/ChenZ24} this is further relaxed by allowing the blocks to be different to some extent. Some works, \eg \cite{DBLP:conf/icml/ChenZ24,pmlr-v139-shamsian21a}, employ so-called hyper networks which orchestrate the updates of the local parameters, etc.

We note that all methods based on parameter sharing suffer from a fundamental limitation -- some local model parts or aspects must remain accessible from outside. Hence, such local models cannot be considered as black boxes, because their design is restricted by the requirement to allow parameter sharing.

Another line of work involves synthetic data generation for sharing the statistical properties of the learned local models and/or local data. We mention especially \cite{DBLP:conf/icml/HuangSZL24}, which is perhaps the closest one to ours (further references can be found therein). Here, each local model consists of two parts: (i) a model for its actual task (e.g.~classification), and (ii) a separate method for synthetic data generation. The generated synthetic data are used by other local models along with their own private data in order to learn their actual tasks. Although \cite{DBLP:conf/icml/HuangSZL24} is based on the same general idea as our approach (i.e.~synthetic data sharing), it has the following weaknesses. First of all, the generative model part is designed in a rather ad-hoc manner by using a general non-learnable distribution matching method for data synthesis. It accounts neither for the actual task specific model nor for the synthetic data from the other local models. Second, although the local models exchange only synthetic data, and can be in principle completely different, they are designed to treat the own and synthetic data in different ways. In particular, the loss terms for the synthetic data are defined on an intrinsic feature space.
% This considerably reduces the choice for local models which actually compute these features.
This considerably reduces the choice for local models to those which actually compute these features.
Finally, the work focuses solely on classification problems and fully supervised learning. Other tasks as well as many practically relevant questions, like \eg semi-supervised learning, multi-modal models, \etc, are not discussed. Our proposed approach addresses and overcomes these weaknesses. 

One of the emergent research directions in FL is multimodal federated learning (see e.g.~\cite{s23156986} for an overview). It aims at learning multimodal data with local models defined on different modal combinations. On the other hand, multimodal and multitask learning is a strong feature of recent generative models like e.g.~(hierarchical) VAEs \cite{Vahdat2020NeurIPS,DBLP:conf/iclr/SutterDV21}, diffusion models \cite{Rombach:CVPR2022}, and foundation models \cite{Radford:Arxiv2021}, to name a few. Our approach establishes a connection between these lines of work. 

To summarize, our contributions are as follows:
\begin{itemize}
    \item We formulate distributed learning in terms of game theory. We prove the existence of a unique Nash equilibrium and convergence of the proposed algorithm to it in the case of exponential family models.
    \item We consider the most general case of model heterogeneity, assuming that the local models are ``black boxes'' that do not need to know anything about each other, except for observing data generated by the others. The local models should be able (i) to optimize their parameters on training data, and (ii) to generate samples from the learned model.
    \item In the case that the local models permit semi-supervised learning, they even do not need to operate on the same probability spaces. This provides a straightforward extension to multimodal data models and seamless data integration.
\end{itemize}

The rest of the paper is organized as follows. The next section introduces a game-theoretic view on distributed learning. It completely resolves the problem of model heterogeneity. For the sake of simplicity, we focus on the case of fully supervised local learning. The following section extends the setup by allowing local models defined on different modalities. For this, we assume that the local models permit semi-supervised learning. The experimental section illustrates the proposed approach on standard benchmark vision datasets for image classification and conditional generation. Finally, the last section discusses further properties, limitations, and extensions of the proposed approach.

\section{Game-theoretic view}\label{sec:game}
To present the core idea of the proposed approach, we consider the case of just one random variable $x$ for the sake of simplicity. Assume we have $N$ local models, i.e.~$N$ parametrized probability distributions $p_{\theta_i}(x)$, $i = 1\ldots N$, where $\theta_i$ represent their parameters (e.g.~network weights). Let $\pi_i(x)$ denote the local training sets (i.e.~empirical target probability distributions for each local model). If each local model has access to all data, the models can be learned by the maximum likelihood principle
\begin{subequations}
\begin{eqnarray}\label{eq:ml1}
    & & L(\theta_i) = \sum_j \alpha_{ij} \sum_x \pi_j(x) \log p_{\theta_i}(x) \rightarrow\max_{\theta_i}, \\
    & & \alpha_{ij}\geq 0, \sum_j \alpha_{ij} = 1 \ \forall i \label{eq:ml1b} ,
\end{eqnarray}
\end{subequations}
where $\alpha_{ij}$ are weights, which assign the ``importance'' of the $j$th training set for the $i$th model.

However, similar to Federated Learning, we assume that each local model can sample only from its own training set. To bypass this restriction without violating it, we assume {\em generative} local models. This means that we can sample realizations $x$ from each $p_{\theta_i}(x)$. The core idea is then to substitute the hidden real data in \eqref{eq:ml1} by synthetic ones, which leads to 
\begin{eqnarray}\label{eq:utility}
    \lefteqn{
    L(\theta_i) = \alpha_{ii} \sum_x \pi_i(x) \log p_{\theta_i}(x) + \nonumber} \\
    & & \sum_{j\neq i} \alpha_{ij} \sum_x p_{\theta_j}(x) \log p_{\theta_i}(x) \rightarrow\max_{\theta_i} \ \ \ \forall i.
\end{eqnarray}
Strictly speaking, the objective $L(\theta_i)$ depends also on all $\theta_j$, $j \neq i$ through samples generated by the other local models. To resolve this, we follow a game-theoretic approach, i.e.~we view \eqref{eq:utility} as player {\em utilities} and optimize them w.r.t.~their own parameters $\theta_i$ only. The parameters $\theta_i$ can then be understood as player {\em strategies} in a {\em cooperative game}.

Next, we discuss the properties of the proposed cooperative game and suggest a simple algorithm for finding its Nash equilibrium. We begin by assuming that the model families $p_{\theta_i}$ have enough expressive power to model the data distributions $\pi_i$ and their mixtures. The optimum of \eqref{eq:utility} w.r.t.~$\theta_i$ is then achieved at
\begin{equation}\label{eq:nash}
    p_{\theta_i}(x) = \alpha_{ii} \pi_i(x) + \sum_{j\neq i} 
    \alpha_{ij} p_{\theta_j}(x) \ \ \forall i ,
\end{equation}
i.e.~when $p_{\theta_i}$ exactly matches the data it is trained on. It follows that the Nash equilibrium (where no player can improve its utility) is the solution of a system of linear equations
\begin{equation}\label{eq:lineq}
    p = \pi A + p B ,
\end{equation}
where $p$ is the vector of all probability distributions $p_{\theta_i}$, $\pi$ is the vector of all training sets $\pi_i$, $A$ is a diagonal matrix with diagonal elements $\alpha_{ii}$, and $B$ is a matrix containing $\alpha_{i\neq j}$ with zero diagonal. We assume that all $\alpha_{ii}$ are strictly positive, because setting an $\alpha_{ii} = 0$ would mean that the $i$th local model ignores its own real data $\pi_i$. This ensures that \eqref{eq:lineq} has a unique solution given by
\begin{equation}\label{eq:solution}
    p = \pi A (\mathbb I - B)^{-1} .
\end{equation}
This solution can also be approached by the following iteration 
\begin{equation}\label{eq:iter}
    p^{(t+1)} = \pi A + p^{(t)} B .
\end{equation}
To prove the existence of the inverse $(\mathbb I - B)^{-1}$ and the convergence of the iterative procedure \eqref{eq:iter}, we employ Gershgorin Circle Theorem, which ensures that all 
eigenvalues $\lambda_k$ of the matrix $B$ are bounded by $|\lambda_k| \leq \max_i (1 - \alpha_{ii})$. It follows that its spectral radius fulfills $\rho(B) = \max_k |\lambda_k| < 1$. Both the uniqueness of the solution of \eqref{eq:solution} and the convergence of \eqref{eq:iter} to it are then a direct consequence.

Equation \eqref{eq:solution} shows that each local model will learn a mixture of the data distributions $\pi_i(x)$. These mixtures are different for each local model because the matrix $(\mathbb I - B)^{-1}$ has rank $N$. It is important to notice that the choice of the $\alpha$-matrix is crucial for {\em both} the resulting local models and the convergence speed. If it is diagonal, \ie $A=\mathbb{I}$, $B=0$, then each local model learns only from its own data and, at the same time, the iteration \eqref{eq:iter} stops after the first step. If, on the other hand, the diagonal elements of the weight matrix are small, then all learned local models can approach similar mixtures (depending on a proper choice of $B$). At the same time, the iteration \eqref{eq:iter} takes long to converge because the spectral radius of the matrix $B$ approaches one in this case.

To move our analysis towards practical cases, we assume that each local model is from an exponential family given by
\begin{equation} \label{eq:exponential}
 p_{\theta_i}(x) = \exp \bigl[\scalp{\phi_i(x)}{\theta_i} - A(\theta_i)\bigr] ,
\end{equation}
where $\phi_i(x)$ is the sufficient statistic, $\theta_i$ is the natural parameter, and $A(\theta_i)$ is the cumulant function. This assumption is fairly general and covers a wide range of practical cases. We prove that, in this case too, the game defined by \eqref{eq:utility} has a unique, asymptotically stable Nash equilibrium.

% \begin{restatable}{theorem}{PropUnique}\label{T1}
%  The cooperative game given by the utilities
%  \begin{equation*}
%     L_i(\theta) = \alpha_{ii} \sum_x \pi_i(x) \log p_{\theta_i}(x) +
%     \sum_{j\neq i} \alpha_{ij} \sum_x p_{\theta_j}(x) \log p_{\theta_i}(x) 
% \end{equation*}
% and strategies given by exponential family distributions \eqref{eq:exponential} 
% has a unique, asymptotically stable Nash equilibrium.
% \end{restatable}
\begin{theorem}\label{T1}
The cooperative game given by the utilities
\begin{equation*}
    L_i(\theta) = \alpha_{ii} \sum_x \pi_i(x) \log p_{\theta_i}(x) +
    \sum_{j\neq i} \alpha_{ij} \sum_x p_{\theta_j}(x) \log p_{\theta_i}(x) 
\end{equation*}
and strategies given by exponential family distributions \eqref{eq:exponential} 
has a unique, asymptotically stable Nash equilibrium.
\end{theorem}
% \end{restatable}

\begin{proof}

Our proof relies on the classic result of \cite{Rosen:Econ1965}, which shows that games satisfying {\em diagonal strict concavity} (DSC), a condition stronger than concavity, have unique Nash equilibria. 

We start by noticing that the cumulant function of an exponential family is convex in its natural parameters. It follows that the game utilities are concave in their own strategies. A sufficient condition for the stronger DSC criterion is that the symmetrised Jacobian of the mapping
 \begin{equation}
  \begin{bmatrix}
   \vdots \\
   \theta_i \\
   \vdots
  \end{bmatrix}
  \mapsto
  \begin{bmatrix}
   \vdots \\
   \nabla_{\theta_i} L_i(\theta) \\
   \vdots
  \end{bmatrix}
\end{equation}
is negative definite. The most convenient way to prove this condition is to ``dualise'' the game. The dual task for maximising $L_i(\theta)$ w.r.t.~$\theta_i$ for exponential families reads
\begin{align}
 & \sum_{x} 
 p_i(x) \log p_i(x) 
 \rightarrow \min_{p_i\geqslant 0} \\
 & \text{s.t.} 
  \begin{cases}
  \sum_x p_i(x) \phi_i(x) = 
  \sum_x \bigl[\alpha_{ii}\pi_i(x) + 
  \sum_{j\neq i} \alpha_{ij} p_j(x) \bigr] \phi_i(x) \\
  \sum_{x} p(x) = 1 .
  \end{cases}
\end{align}
Therefore, we obtain the following ``dual''  game. The strategy of the player $i$ is the distribution $p_i(x)$ and the utility is its entropy $H(p_i)$. This means that the utility of the players depends on their respective strategy only. The game has additional linear constraints, where we assume the existence of an interior feasible point $(p_1,\ldots,p_N)$.

The assertion of the theorem follows from Theorems 3,4,9 in \cite{Rosen:Econ1965}, if we prove that the symmetrised Jacobian of the mapping
 \begin{equation}
  \begin{bmatrix}
   \vdots \\
   p_i \\
   \vdots
  \end{bmatrix}
  \mapsto
  \begin{bmatrix}
   \vdots \\
   - \nabla_{p_i} H(p_i) \\
   \vdots
  \end{bmatrix}
\end{equation}
is positive definite. This is trivial since the Jacobian is diagonal with elements $1/p_i(x)$ in the $i$th diagonal block. 
\end{proof}

Theorems 7-10 in \cite{Rosen:Econ1965} imply then that the simple stochastic gradient ascent algorithm \ref{alg:alg} converges to the unique equilibrium point. Further properties, limitations, and possible extensions of the proposed approach are discussed in Section~\ref{sec:ext}.

\begin{algorithm}
    \caption{\label{alg:alg}Game-theoretic distributed learning.}
    \VS
    \begin{minipage}{0.45\textwidth}

        \VS
        Iterate:
        \begin{enumerate}
            \item Choose a local model $p_{\theta_i}$.
            \item Collect the models $p_{\theta_j}$, $j\neq i$ \\
            (\eg transfer them to the local node).
            \item Repeat several times:
            \begin{enumerate}
                \item Sample the ``local training set'' $\pi'$ from own data $\pi_i$ and models $p_{\theta_i}$ so that proportions of the sampled data correspond to $\alpha_{ij}$:
                \begin{equation}
                    \pi' \sim \alpha_{ii} \pi_i + \sum_{j\neq i} \alpha_{ij} p_{\theta_j} ,
                \end{equation}
                \item Perform a gradient update step with
                \begin{equation}
                    \nabla_{\theta_i} L(\theta_i) = \nabla_{\theta_i}\sum_{x\in\pi'} \log p_{\theta_i}(x) .
                \end{equation}
            \end{enumerate}
        \end{enumerate}
    \end{minipage}
\end{algorithm}

\section{Heterogeneous data collections}
In the previous section, we considered tasks in which all local models are defined on the same set of random variables. Next, we show how to extend the proposed approach to tasks in which each local model's goal is to learn on a subset of random variables. Let $x = (x^1,x^2\ldots x^m)$ be the full collection of random variables. Let $x_i\subset x$ denote the subset of variables the $i$th local model operates with. The $i$th local model's goal is to learn $p_{\theta_i}(x_i)$ for this subset of variables $x_i$. As before, we assume generative models. Let us denote by $x_{ij} = x_i \cap x_j$ the collection of variables which are contained in both $x_i$ and $x_j$. Correspondingly, the utilities \eqref{eq:utility} are modified to
\begin{eqnarray}\label{eq:semisup}
    \lefteqn{L(\theta_i) = \alpha_{ii} \sum_{x_i} \pi_i(x_i) \log p_{\theta_i}(x_i) + \nonumber} \\
    & & \sum_{j\neq i} \alpha_{ij} \sum_{x_j} p_{\theta_j}(x_j) \log p_{\theta_i}(x_{ij}) \rightarrow\max_{\theta_i} \ \ \ \forall i.
\end{eqnarray}
In other words, in addition to its own real data $\pi_i(x_i)$, each local model (a) observes examples of $x_j$ generated by other local models, (b) ignores those components which are not contained in its own set of random variables $x_i$, and (c) optimizes the {\em marginal} likelihood for those variables which are contained in $x_i$. This requires semi-supervised learning, which is understood as the ability of the model to optimize the marginal likelihood $\log p_{\theta_i}(x'_i\subset x_i)$, i.e.~when only a subset of variables is observed.
%  {\color{red} B.F. May be some concluding sentence here?}

\section{Experiments}

\subsection{MNIST}\label{sec:mnist}

Our first experiment aims to illustrate the basic approach \eqref{eq:utility} in the situation when the local datasets (empirical distributions) are considerably different. We use the MNIST dataset for this experiment. We learn probability distributions $p(x)$, where $x\in\{0,1\}^{28\times 28}$ are binarized MNIST images. For modeling $p(x)$ we employ a hierarchical variational auto-encoder (HVAE) with two groups of binary latent variables $z_0\in \{0,1\}^{30}$ and $z_1\in \{0,1\}^{100}$. Hence, the complete decoder model reads $p(z_0, z_1, x) = p(z_0) p(z_1\cond z_0) p(x\cond z_0, z_1)$, where $p(z_0)$ is uniform, and the other terms are implemented as feed-forward networks. The corresponding encoder is $q(z_0, z_1, x) = q(x) q(z_1\cond x) q(z_0\cond z_1,x)$ (i.e.~in so-called reverse factorization order), where $q(x)$ is the data distribution. This HVAE is trained by Symmetric Equilibrium Learning \cite{pmlr-v238-flach24a}\footnote{We adapted the provided code for our needs. The details can be found in the source code \cite{flicsgithub}.}.

In order to validate the proposed approach, we split the training data into two parts: $\pi_{0\ldots 4}$ and $\pi_{5\ldots 9}$, which contain images of digits 0 to 4 and 5 to 9, respectively. Correspondingly, we define two local models, each one observing real data from either the first or second training set. The update iterations for each local model are performed on data mini-batches, which consist of an $\alpha$-fraction with real examples $x$ drawn from the corresponding ``own'' $\pi$, and the remaining $(1-\alpha)$-fraction with complete synthetic examples $(z_0, z_1, x)$ drawn from the decoder of the other model. For the unsupervised examples $x$ from the own data we employ Symmetric Equilibrium Learning. In contrast, when learning on complete synthetic data, we can directly optimise the conditional likelihoods of the encoder and decoder. In addition, we train a single HVAE of the same architecture on the whole MNIST training set as a baseline.

% \begin{figure*}[ht]
%     \begin{center}
%         \includegraphics[width=0.9\textwidth]{figures/generated.png}

%         \VS\VS
%         \includegraphics[width=0.9\textwidth]{figures/fids-mnist.png}
%     \end{center}
%     \caption{\label{fig:mnist}MNIST experiment. Top: images generated by the learned models (we show probabilities $p(x\cond z_0,z_1)$ instead of sampled binary images for better visibility), bottom: dependencies of the obtained FID scores on $\alpha$ (see text for explanation).}
% \end{figure*}

\begin{figure}[ht]
    \begin{center}
        \includegraphics[width=\linewidth]{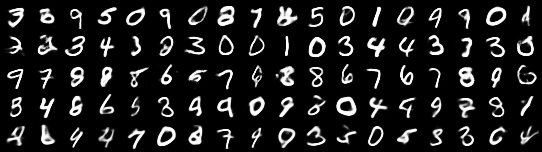}

        \VS\VS
        \includegraphics[width=\linewidth]{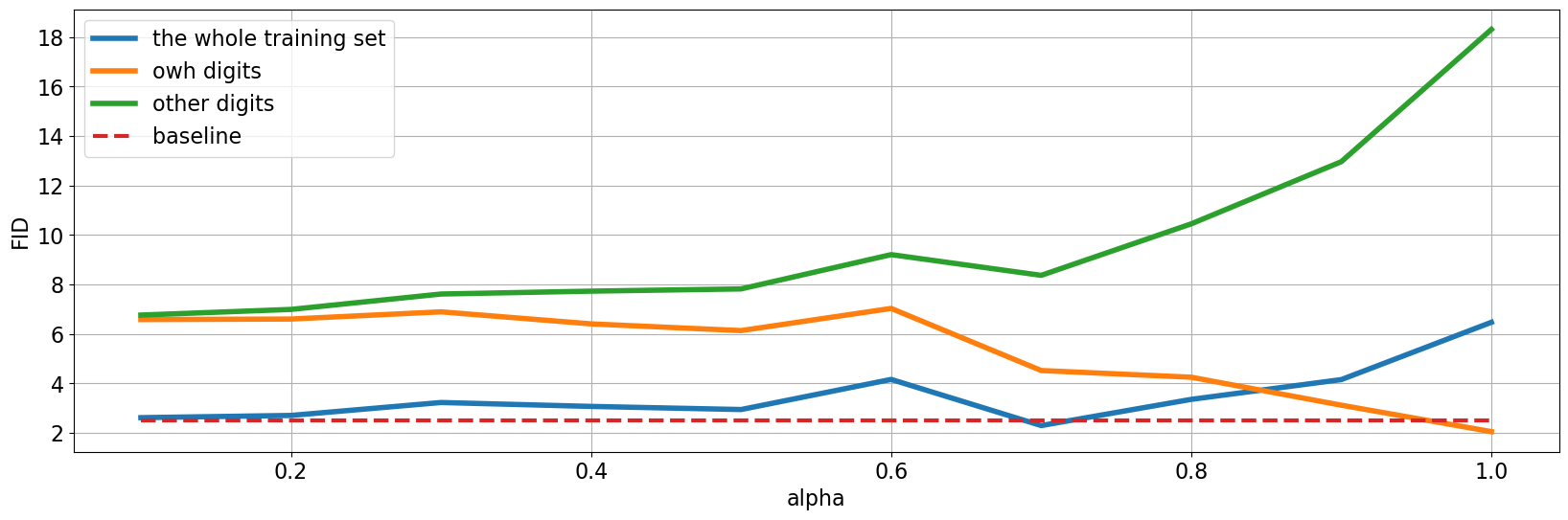}
    \end{center}
    \caption{\label{fig:mnist}MNIST experiment. Top: images generated by the learned models (we show probabilities $p(x\cond z_0,z_1)$ instead of sampled binary images for better visibility), bottom: dependencies of the obtained FID scores on $\alpha$ (see text for explanation).}
\end{figure}

We perform a series of experiments varying $\alpha$ from $0.1$ to $1$. Remember that $\alpha = 1$ means that models are learned on their own real data only, i.e.~independently from each other without any synthetic data. The results are presented in Fig.~\ref{fig:mnist}. The top panel shows images generated from the learned models. The first row corresponds to the baseline HVAE, the next two rows correspond to the pair of HVAEs learned with $\alpha = 1$, and the last two rows show images generated from HVAEs learned with $\alpha = 0.1$. Note that for $\alpha = 1$ only digits from their own datasets are generated, whereas for $\alpha = 0.1$ both local models can generate all digits, although only a subset is presented in their own datasets.

In order to rate models quantitatively, we employ the Fr\'echet Inception Distance (FID) and use the code from \cite{Seitzer2020FID}. In particular, we compute (i) FID scores to the whole training data, (ii) FID scores to the own datasets, and (iii) FID scores to the other datasets (all averaged over both local models). The dependencies of these FID-values on $\alpha$ are shown in the bottom panel of Fig.~\ref{fig:mnist}. Note that the FID scores to the own datasets improve with growing $\alpha$, whereas the FID scores to the other datasets get considerably worse. This is caused by the difference between the target distributions $\pi_{0\ldots 4}$ and $\pi_{5\ldots 9}$. The FID scores to the whole dataset are better (up to stochastic deviations) for smaller $\alpha$ values. This validates the dependence of the results on $\alpha$ discussed in Sec.~\ref{sec:game}.

\subsection{Fashion MNIST}

The next experiment shows that our method can cope with model heterogeneity. We chose the Fashion MNIST dataset for this purpose. We define four local models of different complexity. They have the same layer structure but different numbers of hidden units in their corresponding networks. Each local model has access to its own subset of the training data. Depending on the experiment, a local model learns either a classifier $p(c \cond x)$, where $x$ are images and $c$ are classes, or a joint model $p(x,c)$. In the latter case, the required classifier is a part of the joint model.

\begin{figure*}[ht]
    \begin{center}
    \includegraphics[width=0.8\textwidth]{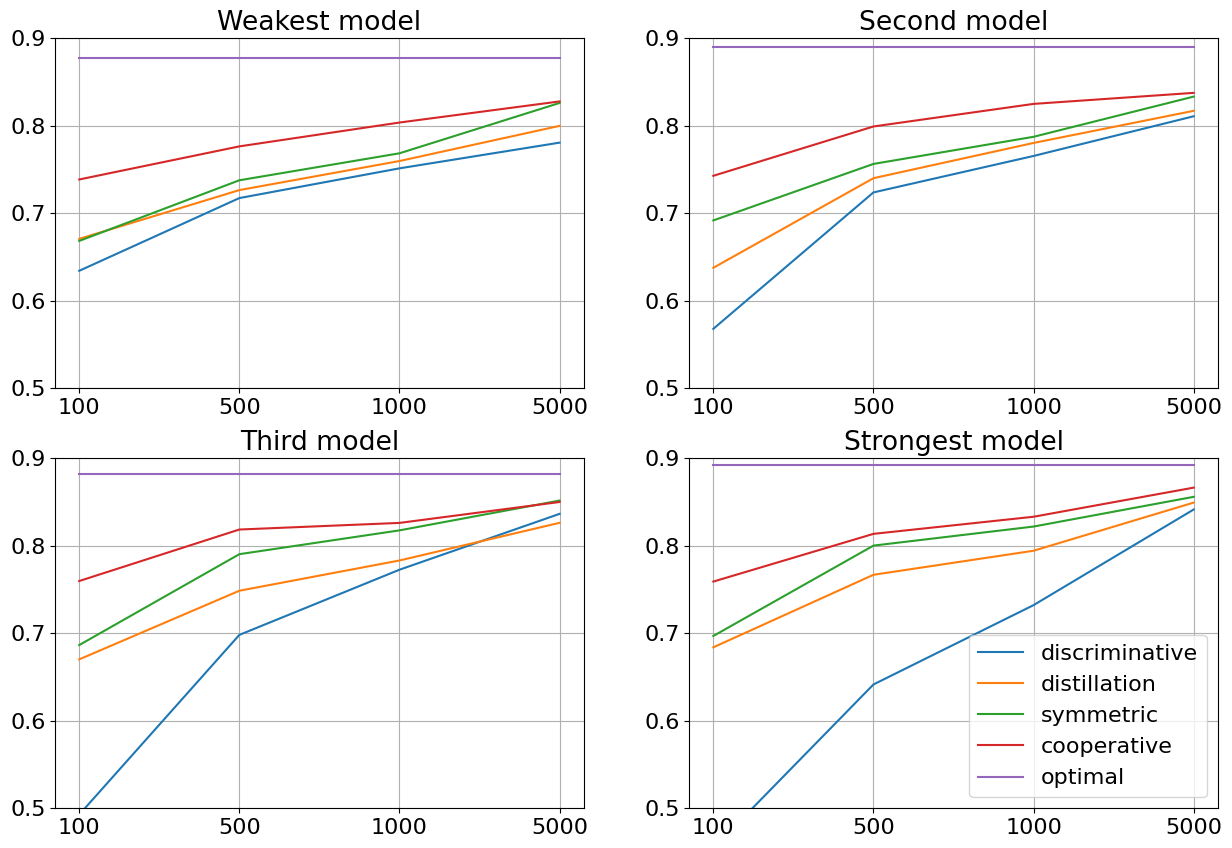}
    \end{center}
    \caption{\label{fig:fmnist}Fashion MNIST experiments. Classification accuracies of the trained models in dependence on the number of training examples. Four panels correspond to four models with different expressive power (see text for explanations).}
\end{figure*}

We performed the following experiments:
\begin{itemize}
    \item[--] {\bf Discriminative} learning. There is no communication between local models, i.e.~they learn independently of each other. Each one learns its own classifier on its own training set. The size of the training sets is growing from 100 to 5000 examples per model, and the datasets for different models do not overlap. We consider this experiment as a baseline. The following experiments aim at improving it either by introducing inter-model communication or learning generatively.
    \item[--] {\bf Distillation}. In addition to its own labeled dataset $(x, c)$, each local model uses labels provided by other models, i.e.~pairs $(x, c')$. Note that this requires no data transfer between the models. It is sufficient if each local model can request inference of other ones on its own dataset.
    \item[--] {\bf Symmetric} learning. There is again no communication between the models. But in contrast to the first experiment, each local model learns a generative model $p(x,c)$. We again use Symmetric Equilibrium Learning \cite{pmlr-v238-flach24a} for this.
    \item[--] {\bf Cooperative} learning. Game-based distributed learning as proposed. Models are learned generatively as in the previous experiment, but in addition to the own dataset $(x,c)$, each local model uses samples $(x',c')$, generated by other ones. Note again that this requires no data transfer between models. They should only be able to call a generating function of the others.
    \item[--] {\bf Optimal}. Classifiers are learned discriminatively on the whole dataset of 60k training images. We consider this experiment as a baseline we would like to reach, i.e.~as a ``gold standard'' which can be reached in principle by classifiers with the used architecture.
\end{itemize}

The results are shown in Fig.~\ref{fig:fmnist}, where the sub-figures correspond to the different local models. They show their validation accuracy as a function of the size of their own training sets. Below are some observations.
\begin{itemize}
    \item[--] All accuracies nicely grow with the size of the training sets.
    \item[--] More powerful models give better results for larger data sets. At the same time, they are slightly more over-fitting, i.e.~the results for small data sets are bad.
    \item[--] Distillation considerably improves the baseline.
    \item[--] Generative learning greatly improves the baseline and is superior to the distillation in all experiments.
    \item[--] Results of all experiments converge to each other with growing data sets. 
    \item[--] The proposed game-based distributed learning approach gives the best results in all experiments.
\end{itemize}

\subsection{PolyMNIST}

% \begin{figure*}
%     \begin{center}
%     \begin{minipage}{0.43\textwidth}
%         \includegraphics[width=0.99\textwidth]{./figures/original.png}

%         \VS
%         \includegraphics[width=0.99\textwidth]{./figures/generated1_11.png}
%     \end{minipage}\hfill
%     \begin{minipage}{0.55\textwidth}
%         \includegraphics[width=0.99\textwidth]{./figures/vfpm.png}
%     \end{minipage}
%     \end{center}
%     \caption{\label{fig:polymnist}PolyMNIST experiments. Left, top: original images, left, bottom: generated images, rows and columns correspond to styles and digits respectively. Right: classification accuracies per style for the considered models, ``avg'' -- averaged over styles.}
% \end{figure*}

\begin{figure}[ht]
    \begin{center}
    \includegraphics[width=\linewidth]{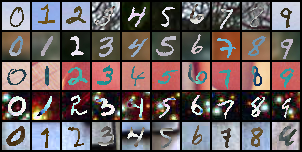}

    \VS
    \includegraphics[width=\linewidth]{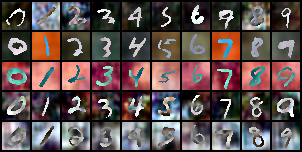}

    \VS\VS
    \includegraphics[width=\linewidth]{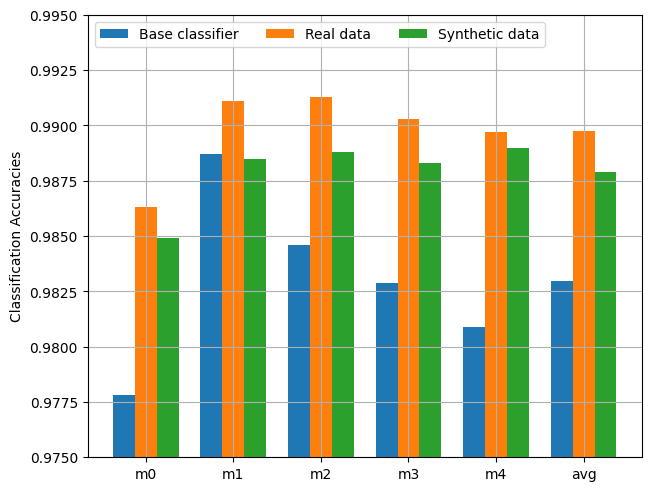}
    \end{center}
    \caption{\label{fig:polymnist}PolyMNIST experiments. Top: original images, middle: generated images, rows and columns correspond to styles and digits, respectively. Bottom: classification accuracies per style for the considered models, ``avg'' -- averaged over styles.}
\end{figure}

Our last experiment demonstrates that the proposed approach is suitable for models that can be trained by semi-supervised learning. This allows to use local models that are defined on different probability spaces. For this experiment, we use the PolyMNIST dataset \cite{DBLP:conf/iclr/SutterDV21}. It is obtained from the original MNIST dataset by coloring MNIST images with different styles (see Fig.~\ref{fig:polymnist} for a sample). The data consist of digits $c\in\{0\ldots 9\}$, styles $m\in\{0 \ldots 4\}$, and RGB images $x\in \mathbb R^{3\times 28\times 28}$.  We notice that PolyMNIST images are produced by using binarized original MNIST images as foreground/background segmentation masks. Therefore, we extend our model by introducing an additional segmentation component $s\in\{0,1\}^{28\times 28}$. Finally, we introduce binary latent variables $z\in\{0,1\}^{16}$. We refer to the Appendix for a detailed definition of the model and its semi-supervised learning. In summary, the model is given by a joint probability distribution $p(c,m,z,s,x)$ over all variables described above. It can be trained by semi-supervised learning, where the training samples are either pairs $(c,s)$ or triplets $(c,m,x)$. The former are given by binarized MNIST images, and the latter are given by PolyMNIST data. The trained model can be used for different inference tasks, in particular, image generation and digit classification.

For the purpose of our experiment, we assume that the MNIST data are private, i.e.~the training pairs $(c,s)$ are not directly accessible. In order to resolve this restriction, we follow our approach and introduce a lightweight generative model $p_{cs}(c, s)$, which learns on binarized MNIST examples and is able to generate MNIST-like synthetic data. The model architecture is very similar to the one used in our first experiment.

We compare two variants. First, when samples $(c,s)$ are given by the MNIST dataset (called {\em real data}), and second, when synthetic samples $(c,s)$ are generated by $p_{cs}(c, s)$. In addition, in order to have a baseline, we learn a simple classifier $p(c\cond x)$ on PolyMNIST images. Its architecture is chosen similar to the corresponding part of the main model. The results are presented in Fig.~\ref{fig:polymnist}. The generated images (in the middle) were produced by sampling from $p(x,z\cond c, m, s)$ with $(c,s)\sim p_{cs}(c,s)$ sampled from the lightweight model, and with randomly chosen style $m$. The bottom panel of the figure shows the achieved digit classification accuracies per style for the considered models. First of all, we observe that the generative model learned with real MNIST data $(c,s)$ outperforms the baseline classifier by a considerable margin\footnote{Note that for such simple data the improvement in a fraction of percent is already essential.}. The improvement is clearly visible especially for the ``hard'' textured styles (like e.g.~``m0'' or ``m4''), where even human observers do not always recognize the presented digit. As expected, learning with synthetic data is slightly worse. Nevertheless, it considerably outperforms the baseline classifier for almost all styles (except the simplest non-textured one, where the baseline classifier already performs well) and on average.

% \subsection{Properties, advantages and limitations.}
% Properties (following \cite{DBLP:conf/icml/HuangSZL24}): heterogeneous, serverless, decentralized ...

% Advantages: in short, black boxes -- nothing is visible outside the local model, neither the architecture, nor some logits, feature spaces, whatever, i.e.~absolutely nothing but synthetic data. Models need not even be NN, etc.

% Discussion about ``generative'' model nature. At the first glance it seems to be a restriction (models should be actually able to generate data). On the other side, the local models should only ``look generatively'' from outside. Depending on a particular application and concrete privacy restrictions, local models can actually be allowed to share some real data, or its parts. Anonymized sharing (only some critical information is hidden). I.e.~local models ``decide by themself'', what and how to share. Other models do not know, whether the observed data is real or synthetic or what for data is actually hidden.

% A fundamental limitation: Inability to learn the average $\pi$, problem with $\alpha{\rightarrow}0$ (btw.~it seems to be a fundamental limitation for all serverless / decentralized approaches, since no global averaging).

% Technical issues: need for ``good'' synthetic examples, initialization problem, possible pre-learning strategies. Learning generative models is time-consuming.

\section{Limitations and extensions}\label{sec:ext}

In this work, we restricted ourselves to the case of strictly positive $\alpha_{ii}$ (see Sec.~\ref{sec:game}) and the discussion of the algorithm’s convergence for this case. However, it also makes perfect sense to consider situations when some of the local models have no private datasets. For instance, if the task is to learn a balanced mixture from all private data (like e.g.~in standard FL with a single global model), this can be achieved by introducing an additional local model without its own data, which learns solely on the synthetic data generated by the other models. This would require extending the convergence analysis to such cases. The same holds for models employing semi-supervised learning (see \eqref{eq:semisup}).

As discussed in Sec.~\ref{sec:game} and experimentally investigated in Sec.~\ref{sec:mnist}, the choice of the $\alpha$-matrix is crucial for both the resulting models and the convergence speed. In practice, the requirements can contradict, for example, if the locally learned probability distributions should be as similar as possible, keeping at the same time fast convergence. This trade-off can be approached by allowing dynamic $\alpha$-matrices. We could \eg start learning with large $\alpha_{ii}$ values to enforce convergence, and then fine-tune with small $\alpha_{ii}$ values to achieve the necessary properties of the results. Again, this would require extending the convergence analysis to such cases with non-stationary $\alpha$-matrices.

Our work assumes generative models. At first glance, this may appear to be a restriction (models should actually be able to generate data). In practice, however, this can be often relaxed to local models that only appear to be generative from outside. For instance, depending on the particular application and privacy restrictions, a local model may actually be allowed to share some parts of its data. Consider \eg a local model over pairs $(x,y)$, with the primary goal to predict labels $y$ from observations $x$. Its local dataset consists of pairs $(x^\ast,y^\ast)$. Assuming that the private observations $x^\ast$ can actually be shared, the model can generate ``synthetic'' data $(x^\ast, y')$ with $y'\sim p(y\cond x^\ast)$ sampled from the learned classifier. Hence, no generative model is needed in this case. Quite often, private data can be shared but should be anonymized, i.e.~some data parts have to remain hidden. Note that the other local models ``do not need to know'', whether the observed data are real or synthetic, or whether some parts are actually hidden. Such scenarios fit perfectly into the proposed game-theoretic approach.

\section{Conclusion}

We propose to treat distributed learning in terms of game theory, i.e.~we view it as a cooperative game. We consider local models as players, each having its strategy represented by its learnable parameters. Each player optimizes its own utility. This representation completely resolves the model heterogeneity issue -- local models are considered as black boxes, being able (i) to learn their parameters from data, and (ii) to generate synthetic data. For the latter, we assume generative local models. We prove the existence of a unique Nash equilibrium for exponential family local models as well as the convergence of the proposed learning approach to this equilibrium. We experimentally show that our approach can effectively mix local empirical data distributions given by local training sets without transferring real data between local models. Moreover, our approach allows learning heterogeneous local models, even if they are defined on different probability spaces. Given the promising results, we see numerous directions for future work, some of which are discussed in Sec.~\ref{sec:ext}.

\bibliographystyle{IEEEtran}
\bibliography{dpplgm}

% \clearpage

\appendix

\subsection{Network architectures}

All conditional distributions used in our experiments are implemented as feed-forward networks. We use the following network architectures:

\begin{itemize}
    \item[--] {\bf Encoders} are networks of spatially decreasing resolution, which have 2D inputs (e.g.~images) and 1D outputs. An example of such a network is the classifier $p(c\cond x)$ used in the Fashion MNIST experiments. Our encoders have 6 convolutional layers, some of which with strides larger than one, to effectively reduce the spatial resolution. The number of channels increases with decreasing spatial resolution, typically starting from 16 after the first convolution.
    \item[--] {\bf Decoders} have a reverse architecture, i.e.~they are used to produce 2D outputs from 1D inputs. They have the same number of layers and channels as the corresponding encoders, but are implemented by using transposed convolutions. An example of such a decoder is $p(x\cond z_0,z_1)$ used in the MNIST experiments (1D inputs $z_0$ and $z_1$ are concatenated).
    \item[--] {\bf MLP}s typically consist of 3 fully connected layers. The number of units in the hidden layers is usually twice as large as the dimension of the input or of the output. An example is $p(z_1\cond z_0)$ in the MNIST experiments.
\end{itemize}
We use $\tanh$ activations in all our networks and optimize them using the Adam optimizer with gradient step size $1e{-}4$. 

Regarding the training times: The most time-consuming experiment is the training of the PolyMNIST model with synthetic data, i.e.~learning of the ``main model'' using samples from the ``lightweight model'', which takes about 38 hours on a single NVIDIA H100 GPU. Learning the Fashion MNIST model takes about 24 hours. Note, however, that in this experiment, 4 models are learned simultaneously on a single GPU. MNIST experiments take about 15 hours per run (2 models learned in each run).

\subsection{Latent variable models}

Next, we explain in more detail how we deal with latent variables on the example of $p(x)$ used in the MNIST experiments. Remember that we use a HVAE, where $p(z_0,z_1,x) = p(z_0)p(z_1\cond z_0)p(x\cond z_0, z_1)$ is a hierarchical decoder, and $q(z_0,z_1,x)=q(x)q(z_1\cond x)q(z_0 \cond z_1, x)$ is a hierarchical encoder with reverse factorization order. The distribution $p(z_0)$ is assumed uniform and $q(x)$ is the data distribution. All other components are implemented as feed-forward networks as described above. We use Symmetric Equilibrium Learning \cite{pmlr-v238-flach24a}, which consists in this case of the following steps (given a training example $x$):
\begin{enumerate}
    \item Sample from the encoder, i.e.~$z_1\sim q(z_1\cond x)$ followed by $z_0\sim q(z_0\cond z_1, x)$, optimize the decoder's likelihoods $\log p(z_1 \cond z_0)$ and $\log p(x\cond z_0, z_1)$.
    \item Sample from the decoder, i.e.~$z_0\sim U$, $z_1\sim p(z_1\cond z_0)$, $x\sim p(x\cond z_0, z_1)$, optimize the encoder's likelihoods $\log q(z_1 \cond x)$ and $\log q(z_0\cond z_1, x)$.
\end{enumerate}

In the Fashion MNIST experiments, we define $p(x\cond c)$ as a {\em conditional} HVAE with two groups of latent variables. The only difference to the MNIST case described above is that all constituents of both the encoder and decoder are additionally conditioned on $c$. The latent variables in the PolyMNIST experiments are handled in a somewhat different way (see below).

%%%%%%%%%%%%%%%%%%%%%%%%%%%%%%%%%%%%%%%%%

\subsection{PolyMNIST model}

The main model in the PolyMNIST experiments is a joint probability distribution over the following random variables: digits $c\in\{0\ldots 9\}$, styles $m\in\{0 \ldots 4\}$, RGB images $x\in \mathbb R^{3\times 28\times 28}$, binary segmentations $s\in\{0,1\}^{28\times 28}$, and binary latent variables $z\in\{0,1\}^{16}$. We model the joint probability distribution $p(c,m,z,s,x)$ by means of the conditional probability distributions of its components conditioned on the rest, e.g.~$p(x\cond c,m,z,s)$, $p(c\cond m,z,s,x)$, etc. This choice is motivated by the following considerations. We want to address semi-supervised learning in an EM-like manner. This means that we need to complete partial observations, say $(c,m,x)$ (the PolyMNIST data), by sampling the remaining variables, in this case $(z',s')\sim p(z,s\cond c,m,x)$. The joint log-likelihood of the completed sample $\log p(c,m,z',s',x)$ is then optimized as if it were a completely supervised training example. Other combinations of observables would then require other ``auxiliary'' conditional probability distributions for their completion. The proposed representation of the joint distribution by means of its conditionals allows us to implement all conditional probability distributions that are required for the completion of partial observations. Sampling from any conditional probability distribution can be performed by Gibbs sampling with clamped observed parts. The same procedure can be also used for inference, for example, for predicting digits $c$ from images $x$. In this case, we employ the maximum marginal decision strategy, i.e.~we need to compute $\arg\max_c p(c\cond x)$ (i.e.~marginalized over all other variables). We run Gibbs sampling with clamped $x$ (i.e.~iterating over all conditionals except $p(x\cond\cdot)$), accumulate digit frequencies, and decide at the end for the most frequent one.

Learning the conditional distribution models from completed examples corresponds to pseudo-likelihood learning \cite{Besag:1975}. We note, however, that completing partial observations by Gibbs sampling is valid only if all conditional probability distributions are consistent \cite{Shekhovtsov:ICLR2022}. Strictly speaking, this consistency cannot be guaranteed because the conditional distributions are modeled independently of each other and learned from data. On the other hand, the training data in our case are abundant (60k segmentations in MNIST and 300k images in PolyMNIST). Assuming completely supervised learning, i.e.~complete training samples $(c,m,z,s,x)$, and powerful enough models, pseudo-likelihood learning should recover the true conditionals, which are then obviously consistent.

To summarize, our learning algorithm optimizes the pseudo-likelihood, where incomplete training examples are completed by Gibbs sampling. In particular, we have two types of training samples: $(c,s)$ (given in MNIST) and $(c,m,x)$ (given in PolyMNIST). The former are completed by iterating sampling from $p(m\cond \cdot)$, $p(z\cond \cdot)$, and $p(x\cond \cdot)$, starting from random $m$, $z$, and $x$. For the latter, analogously, we iterate sampling from $p(z\cond \cdot)$ and $p(s\cond \cdot)$. Note that we have no training examples where both image and segmentation are given simultaneously.

The feed-forward architectures used in this experiment slightly differ from the ones described above, since here we often need to mix 1D and 2D inputs. The networks are either directly composed of our standard ``building blocks'' or correspondingly adapted. For example, the conditional distribution for segmentations, i.e.~$p(s\cond \cdot)$, is implemented as a UNet-like architecture, where the encoder- and decoder-branches have the same architecture as our encoders and decoders in the previous experiments. Further details can be found in the source code \cite{flicsgithub}.

\end{document}